\documentclass{article} 
\usepackage{iclr2025_conference,times}

\usepackage{amsmath,amsfonts,bm}
\usepackage{amssymb}

\usepackage{amsthm}

\newtheorem{definition}{Definition}
\newtheorem{assumption}{Assumption}
\newtheorem{proposition}{Proposition}

\newcommand{\ind}{\perp\!\!\!\!\perp} 

\newcommand{\cD}{\mathcal{D}} 
\newcommand{\cT}{\mathcal{T}}

\def\eqref#1{equation~\ref{#1}}

\DeclareMathAlphabet{\mathsfit}{\encodingdefault}{\sfdefault}{m}{sl}
\SetMathAlphabet{\mathsfit}{bold}{\encodingdefault}{\sfdefault}{bx}{n}

\def\sP{{\mathbb{P}}}

\def\sR{{\mathbb{R}}}

\newcommand{\E}{\mathbb{E}}

\usepackage{hyperref}
\usepackage{url}
\usepackage{nicematrix}
\usepackage{algorithm}
\usepackage{algpseudocode}
\usepackage{amsmath}
\usepackage{bbold}

\usepackage{graphicx}
\usepackage{subcaption}

\usepackage{booktabs}
\usepackage{multirow}

\title{\centering Training and Evaluating Causal Forecasting Models for Time-Series}

\author{Thomas Crasson \& Yacine Nabet \thanks{Joint first authors. Work done in part while TC was at UBC.} \\
Wiremind\\
\texttt{thomas.crasson@polytechnique.edu, ynabet@wiremind.io}\\
\And
Mathias L\'ecuyer \\
University of British Columbia \\
\texttt{mathias.lecuyer@ubc.ca}\\
}

\usepackage{dsfont}
\newcommand{\1}{\mathds{1}}

\iclrfinalcopy 

\begin{document}

\maketitle

\begin{abstract}
Deep learning time-series models are often used to make forecasts that inform downstream decisions.
Since these decisions can differ from those in the training set, there is an implicit requirement that time-series models will generalize outside of their training distribution.
Despite this core requirement, time-series models are typically trained and evaluated on in-distribution predictive tasks.
We extend the orthogonal statistical learning framework to train causal time-series models that generalize better when forecasting the effect of actions outside of their training distribution. To evaluate these models, we leverage Regression Discontinuity Designs popular in economics to construct a test set of causal treatment effects.
\end{abstract}

\section{Introduction}
\label{sec:intro}
Deep learning models have recently seen major improvements, including in time-series modelling \cite{itransformer,PatchTST,TSMixer,Tide,zhang2023crossformer}.
Time-series models are often used to make forecasts that inform decisions, such forecasting demand for goods or transportation to optimize pricing, or forecasting health markers to optimize treatment decisions \cite{MAKRIDAKIS20221346,Nowroozilarki_2021,makridakis2023m6forecastingcompetitionbridging}.

The decision making procedures that leverage time-series forecast typically optimize for a given outcome, such as revenue or health outcome, and lead to actions that differ from those observed in the training data.
This results in a key implicit requirement on time-series models: that they will generalize outside of the observational distribution.
Despite this requirement, the only complex tasks on which time-series forecasting models are trained and evaluated are in-distribution predictive tasks \cite{itransformer,causal_transformers,causal_forecasting,PatchTST,TSMixer,Tide,zhang2023crossformer,Nowroozilarki_2021,MAKRIDAKIS20221346,makridakis2023m6forecastingcompetitionbridging}.

In this work, we develop a training and evaluation procedure for causal forecasting time-series models. Causal models aim to learn the causal relationship between a treatment (e.g., price) and its effect on a specific outcome (e.g., demand), conditioned on other, observational features.
That is, they aim to capture the change in outcome caused by a change in treatment, and not just predict the outcome from an in-distribution observed treatment.
As a result, they generalize better when forecasting the effect of actions outside of their training distribution.

A few causal forecasting models have been proposed,
but they use approaches that do not map to well defined causal effects  \cite{adversarial_loss,timeseriesdeconfounderestimating,causal_transformers}, do not correct for regularization bias \cite{li2020gnetdeeplearningapproach,adversarial_loss,timeseriesdeconfounderestimating,causal_transformers,causal_forecasting}, or take inspiration from causal frameworks for heuristics \cite{Brodersen_2015,schultz2024causalforecastingpricing}. 
In this work, we leverage the orthogonal statistical learning framework of \citet{foster2023orthogonal} to learn a causal forecasting model. We extend this framework to a well-defined time-series causal problem and to high-dimensional treatments, and instantiate it on top of state-of-the-art backbone time-series models to support complex use-cases.

Evaluating causal forecasting models is even more challenging, as outcomes under alternative treatments are by definition unobserved. Existing work thereby rely on simple simulations with known counterfactuals \cite{adversarial_loss,timeseriesdeconfounderestimating,li2020gnetdeeplearningapproach,causal_forecasting}, or in-distribution forecasting performance \cite{causal_forecasting,causal_transformers}. The former is often too simplistic to really compare models, while the later can be misleading and inflate the performance of non-causal models. We show details of this effect in \S\ref{sec:causal-background}.

To evaluate forecasting models on causal tasks, we take inspiration from Regression Discontinuity Designs (RDDs) popular in econometrics \cite{hahn2001identification,imbens2008regression}, to construct a test set of causal effects estimated under different assumptions than the typical model-fitting assumption. We can then compare the treatment effects estimated by forecasting models using traditional evaluation metrics.
We train and evaluate our causal forecasting models on two relevant tasks: demand forecasting using a proprietary dataset of passenger rail prices and demand over time; and the public health dataset MIMIC-III  \cite{MIMIC-III}.
We show that, compared to traditional time-series forecasting models and a state-of-the-art causal model, our causal forecasting models predict causal effects that are $36\%$ and $1\%$ closer to those estimated with RDDs, respectively.


\section{Causal Forecasting with Orthogonal Learning}
\label{sec:causal-forecasting}
In the reminder of this paper, we use a demand forecasting task for passenger rails as a running example. For each scheduled train, models aim to predict the daily number of seats sold at a given price using temporal information and other meta-data (see \S\ref{sec:eval-wm} for details). The end goal of these forecasts is to set prices to maximize revenue.
In \S\ref{sec:eval}, we evaluate our models on this task using a proprietary dataset, as well as on a public health forecasting dataset \cite{MIMIC-III}.

We abstract our forecasting task as a dataset of $N$ observed time-series ($N$ trains) indexed by $n$, with time indexed by $t$. Each time-series $n$ consists in stationary context features $S^n$, temporal (non-stationary) features measured at each time step $X^n_{1:t}$, a treatment at each time step $T^n_{1:t}$, and real valued outcomes $Y^n_{1:t}$ ($Y^n_t \in \sR$). The objective is to predict each outcome $Y^n_t$ using static, temporal, and treatment features, as well as past outcomes up to a time $\tau < t$ ($Y^n_{\tau+1:t-1}$ is unknown).
We thus seek to learn $\hat{Y}^n_t \triangleq f(S^n, X^n_{1:t},T^n_{1:t},Y^n_{1:\tau})$, such that $\hat{Y}^n_t$ is close to $Y^n_t$. This setting differs from standard time-series forecasting tasks \cite{itransformer}, but is more relevant to causal tasks (\S\ref{sec:causal-forecasting}).

We next formalize causal effects, show their importance to our task, and introduce orthogonal learning theory that we leverage in our models (\S\ref{sec:causal-background}). We then extend orthogonal learning to time-series models (\S\ref{sec:causal-forecasting-theory}), and instantiate this theory with deep learning architectures (\S\ref{sec:causal-forecasting-models}).

\subsection{Background: Causal Treatment Effects and Orthogonal Learning}
\label{sec:causal-background}

To formalize the importance of causality in forecasting for optimizing decisions, we first need to introduce treatment effects. In the reminder of this paper, we use the potential outcomes (or Rubin's) model of causal inference \cite{holland1986statistics,10.5555/2764565}, and flexible data models from the causality \cite{robinson,Nie_Wager,chernozhukov2018double}.

{\bf Potential Outcomes} represent the values of an outcome under different possible treatments. Consider our demand forecasting setting with two possible prices (treatments) $T_1$ and $T_2$. For a given time-series $n$ and time step $t$, we define the potential outcomes $Y^n_t(T_1)$ and $Y^n_t(T_2)$ as the value of the demand $Y^n_t$ we would observe under each price. Of course, we only ever observe one potential outcome, as $Y^n_t = \1_{T_1} Y^n_t(T_1)+\1_{T_2} Y^n_t(T_2)$ where $\1$ is the indicator function.

{\bf Treatment Effects} represent the causal effect of a change in treatment (price) on the outcome of interest (demand). Formally, the treatment effect of going from price $T_1$ to $T_2$, for time-series $n$, at time step $t$, is $Y^n_t(T_2) - Y^n_t(T_1)$.
The fundamental challenge of causal inference is that we can only ever observe one potential outcome, for the treatment we chose. As a result, we cannot train a supervised treatment effects model. However, we can estimate expected causal effects under unconfoundedness, the main assumption in the literature which we make when learning our models:

\begin{assumption}[Unconfoundedness]\label{eq:assumption-uncounfoundedness} Consider a set of possible treatments $\cT$. We say that the time-series $(Y_t(T_t), T_t, S, X_t)$ are (conditionally) uncounfounded if, for every time-step $t > \tau$:
\[
\{Y_t(T_t),~T_t \in \cT \} \ind T_t \mid S,~X_{1:t},~T_{1:\tau},~Y_{1:\tau}
\]
\end{assumption}
Intuitively, this assumption means that in sub-populations defined by the conditioning variables, the groups receiving each treatment are comparable (have the same potential outcomes distribution).
Unconfoundedness is verified by construction in (conditionally) randomized trials, as treatments are assigned independently of potential outcomes.
It is also verified when context variables $S,~X_{1:t},~T_{1:\tau},~Y_{1:\tau}$ include all pre-treatment variables that are common causes of $T^n$ and $Y^n$ \cite{pearl2009causality}. In our running example, this consists in all variables that influence both demand and pricing decisions, such as the cities of arrival and departure, day of the trip, and more. Under Assumption~\ref{eq:assumption-uncounfoundedness}, a well known identifyability result holds (see e.g., \cite{10.5555/2764565,pearl2009causality}):

\begin{proposition}[CATE Identifyablity]\label{prop:cate}
Under Assumption~\ref{eq:assumption-uncounfoundedness}, we can identify the Conditional Average Treatment Effect (CATE) as follows:
\begin{align*}
    \textrm{CATE} & \triangleq \E\big[Y_t(T_2) - Y_t(T_1) \mid S,~X_{1:t},~T_{1:\tau},~Y_{1:\tau} \big] \\
    & = \E\big[Y_t \mid T_t=T_2,~S,~X_{1:t},~T_{1:\tau},~Y_{1:\tau} \big] - \E\big[Y_t \mid T_t=T_1,~S,~X_{1:t},~T_{1:\tau},~Y_{1:\tau} \big] .
\end{align*}
\end{proposition}
This important result shows that we can estimate the Conditional Average Treatment Effect (CATE) from conditional expectations $\E\big[ Y_t \mid T_t=T,~S,~X_{1:t},~T_{1:\tau},~Y_{1:\tau} \big]$. This is a familiar quantity to ML practitioners, as a flexible forecasting model $f: T, S, X, T_{1:\tau}, Y_{1:\tau} \rightarrow Y$ trained with a mean-squared error will converge to this conditional expectation in the limit of infinite data. Estimating the CATE $\E\big[Y_t(T_2) - Y_t(T_1) \mid S,~X_{1:t},~T_{1:\tau},~Y_{1:\tau} \big]$ is then as simple as predicting on the two treatment values $f(T_t=T_2,\ldots) - f(T_t=T_1,\ldots)$.
Unfortunately, practice is more challenging.

{\bf The need for causal forecasting} arises when fitting complex forecasting models with finite sample sizes. In observational data collected from regular interactions, such as observed demand for passenger rail seats during regular operation, some treatments will be over-represented. For instance, in our data we observe high prices when demand is high, and low prices when demand is low. Of course, we would expect higher prices to cause a {\em decrease} in demand. But the correlation in the data is flipped, as the transportation company operators use their experience and domain knowledge about demand to set higher prices when underlying demand is high, thereby increasing revenue.

\begin{figure}[h]\
\begin{subfigure}[b]{.3\textwidth}
  \centering
  \includegraphics[width=\linewidth]{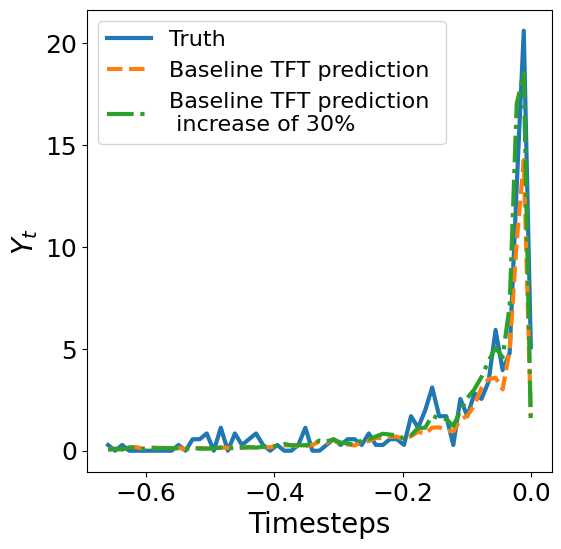}
  \caption{TFT}
  \label{fig:motivation-train-example-TFT}
\end{subfigure}
\hfill
\begin{subfigure}[b]{.3\textwidth}
  \centering
  \includegraphics[width=\linewidth]{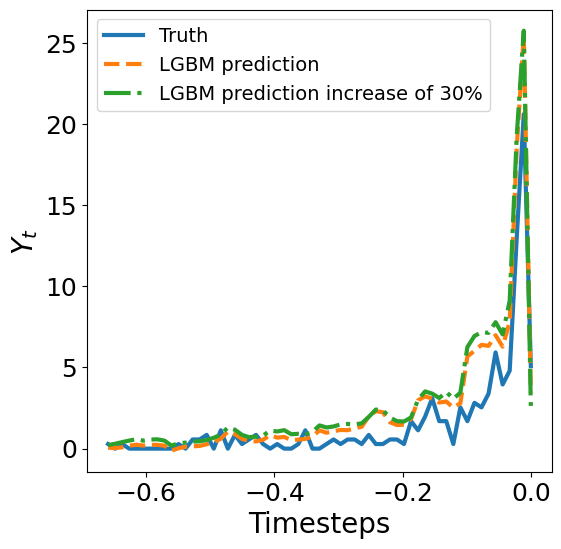}
  \caption{LGBM}
  \label{fig:motivation-train-example-LGBM}
\end{subfigure}
\hfill
\begin{subfigure}[b]{0.3\textwidth}
    \centering
    \includegraphics[width=\linewidth]{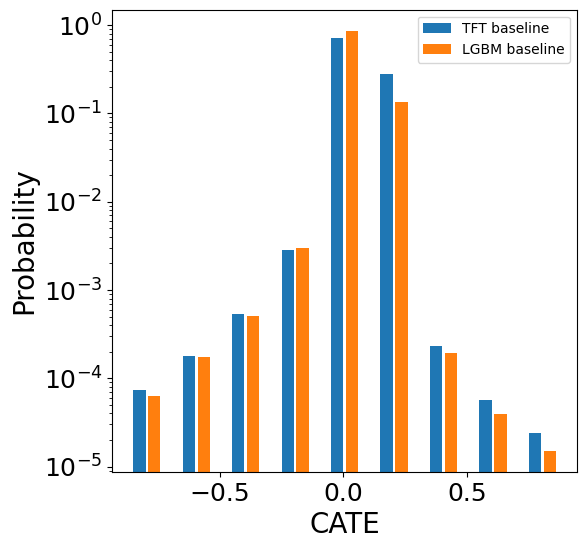}
    \caption{Treatment effects distribution}
    \label{fig:motivation-elasticity}
\end{subfigure}
\vspace{-0.5em}
\caption{Forecasts from a LightGBM and a TFT model. (a),(b): Daily demand forecast change under price increase for one train. The blue line is the ground truth; the orange line is the model's forecast for observed prices; the green curve is the forecast under a 30\% price increase. Both models predict an increase in demand under increased prices. (c): CATE distribution over the test set, for a change from the observed price to the next possible price, normalized by the price change.}
\label{fig:motivation}
\end{figure}

When trained on finite data with such biased treatment assignments, forecasting models will extrapolate based on feature and treatment values. This introduces bias in treatment effect estimates when models are not flexible enough to represent the true function, or due to regularization (both explicit regularization such as weight decay, or deep learning models' implicit regularization \citep{soudry2018implicit,ji2020gradient}). This effect, called regularization bias, happens even under simple treatment effect functional forms \citep{chernozhukov2018double}.
In practice, this is exactly what we observe when fitting ML models on our demand forecasting task.
Figure \ref{fig:motivation} shows predicted treatment effects from a well tuned LightGBM \cite{ke2017lightgbm} boosted tree and an improved version of the Temporal Fusion Transformer (TFT) \cite{TFT} deep learning model. Both models have comparable in distribution performance, with a small edge for the LGBM model (details in \S\ref{sec:eval}). Figures \ref{fig:motivation-train-example-TFT} and \ref{fig:motivation-train-example-LGBM} show predictions for a specific train: both models predict demand quite accurately using the realized price trajectory (orange line, $T^n_{1:t}$ is set to the observed prices). However, when increasing the prices by 30\% for all $t$ (green line), both models predict an {\em increase} in demand. This is the correlation we observe in the data, of opposite sign to the causal relationship we expect. Figure \ref{fig:motivation-elasticity} shows that this affects a significant proportion of test set trains: the LGBM and TFT models respectively predict 13\% and 28\% of positive treatment effects.

\begin{algorithm}[H]
    \caption{Sample Splitting Learning Algorithm from \citet{foster2023orthogonal}}
    \begin{algorithmic}

    \State

    \State \textbf{Input:} Sample set $S = z_1, \ldots, z_n$

    \State Split S into subsets $S_1 = z_1, \ldots, z_{\lfloor n/2 \rfloor}$, and $S_2 = S \backslash S_1$
    \State Fit $g$ for $g_0$ on $S_1$
    \State Fit $\theta$ for $\theta_0$ on $S_2$ using $g$ for $g_0$ in the loss
    \State \textbf{Return} $g, \theta$
    \end{algorithmic}
    \label{alg:orthogonal-learning}
\end{algorithm}

{\bf Orthogonal Learning} \citep{foster2023orthogonal} is a recent learning theory framework to study the convergence rate of models with nuisance parameters, which can be used to estimate CATE models. The framework is very general: we next describe the core results relevant our work, that we later extend to our setting (\S\ref{sec:causal-forecasting-theory}). The results for treatment effect estimation consider a standard binary treatment effect data model \citep{robinson,Nie_Wager,chernozhukov2018double}, with observations $z=(W, Y, T)$ from distribution $\cD$, and $T \in \{0, 1\}$, such that:
\begin{align}
  &  Y = T \cdot \theta_0(W) + f_0(W) + \epsilon_1 , & \mathbb{E}[\epsilon_1|W,T] = 0 \label{eq:data-model-binary-y}\\
  &  T = e_0(W) + \epsilon_2 , & \mathbb{E}[\epsilon_2|W] = 0 \label{eq:data-model-binary-t}
\end{align}
Further define $m_0(w) \triangleq \E[Y | W=w]$. We call $g_0 \triangleq (e_0, m_0)$ the nuisance parameters, required to learn the CATE $\theta_0$ but not direct objects of interest in this framework. In this model $\theta_0(W)$ is the CATE (as it is turned on and off by $T$, see Proposition \ref{prop:cate}).
This data model is quite flexible, as the CATE depends arbitrarily on the features $W$ through $\theta_0$, as do the outcome $Y \in \sR$ and treatment $T$ through $f_0, e_0$. Note that under Assumption \ref{eq:assumption-uncounfoundedness}, this decomposition is without loss of generality.

Calling $\theta, e, m$ models for $\theta_0, e_0, m_0$, respectively, we can define the residualized loss (R-loss) $l(\cdot)$, and its population equivalent $L_{\cD}(\cdot)$, as:
\[\label{eq:def-r-loss}
l(\theta, m, e; z) = \Big( Y - m(W) - \big(T - e(W)\big)\theta(W) \Big)^2 \ \ ; \ \ L_{\cD}(\theta, m, e) = \E_{z \sim \cD} [l(\theta, m, e; z)].
\]

Now consider a model $g$ for $g_0$ learned on sample $S$, such that $\|g - g_0\|_{L_2} \triangleq \sqrt{\E_{z\sim \cD} \|g(w) - g_0(w)\|^2_2} \leq \textrm{Rate}_{\cD}(S, \delta)$ w.p. at least $1-\delta$. Since $m_0(w) = \E[Y | W=w]$ and $e_0(w) = \E[T | W=w]$, these can be models that minimize the mean squared error. 

Further consider an estimation algorithm of model $\theta$ for $\theta_0$ which, given a model $g$ for nuisance parameters, outputs $\theta \in \Theta$ using sample $S$ (e.g., by minimizing the empirical loss $l(\cdot)$ over $S$) such that $L_\cD(\theta, g) - L_\cD(\theta^*, g) \leq \textrm{Rate}_\cD(S, \delta; \theta, g)$, where $\theta^* \triangleq \arg\min_{\theta \in \Theta} L_\cD(\theta, m_0, e_0)$ is the best predictor for $\theta_0$ in class $\Theta$.
A low risk $L_\cD(\theta, g) - L_\cD(\theta^*, g)$ does not indicate a good model $\theta$, as the loss $l$ is computed with $g$, and not $g_0$ the true nuisance parameter.
However, \citet{foster2023orthogonal} show (in Theorem 2, and Appendix J.2, with a constant we derive in Appendix \ref{sec:appendix_theory}) that:
\begin{proposition}[Orthogonal Learning for Binary Treatment Effect]\label{prop:orthogonal-learning-binaray}
Consider a class of functions $\Theta$ where $\forall W, \ \forall \theta \in \Theta: \ |\theta(W)| \leq M$.
Fitting $g$ and $\theta$ with the procedures above using Algorithm \ref{alg:orthogonal-learning}, we have that with probability at least $1-\delta$:
\begin{align*}
\E\big[ \big(T - e_0(W)\big)\big(\theta(W) - \theta^*(W) \big) \big]^2 & \leq L_\cD(\theta, g_0) - L_\cD(\theta^*, g_0) \\
     & \leq \textrm{Rate}_\cD(S_2, \delta/2; \theta, g) + 2(1+M^2) \cdot \textrm{Rate}_{\cD}(S_1, \delta/2)^2 .
\end{align*}
\end{proposition}
There are three key takeaways. {\bf (1)} we can get a CATE model $\theta$ close to the best possible predictor in the class (which can be $\theta_0$ if $\theta_0 \in \Theta$) for all values $W$ with treatment variation around the mean. This is similar to the positivity assumption ($\forall T, W: 1 > \sP(T | W) > 0$) pervasive in causal inference results. {\bf (2)} the rate at which $\theta$ converges to $\theta^*$ under estimated nuisance parameters $g$ is sufficient to bound the error under the true parameters $g_0$. {\bf (3)} the rate of estimation of the nuisance parameters is squared: Algorithm \ref{alg:orthogonal-learning} is tolerant of errors in estimating complex nuisance parameters. For instance, a rate of convergence of $1/\sqrt{N}$ for $\theta$, is preserved if $g$ converges as $o(1/N^{\frac{1}{4}})$.

\subsection{Orthogonal Learning for Time Series Models}
\label{sec:causal-forecasting-theory}

We adapt the orthogonal statistical learning framework described above by making two key extensions: defining daily treatment effects with an observation cut-off; and extending the R-loss to categorical and linear effects with various encodings for $\theta$'s predictions.

{\bf Treatment effects in time series.} The first step is to formalize the treatment effects to predict, and the context to predict them with. In this paper, we focus on {\em daily treatment effects} defined in Proposition \ref{prop:cate}. While technically past prices are likely to causally impact demand at $t$, domain knowledge tells us that the price at $t$ is the dominant factor. This modelling choice lets us focus on single-day treatments, though supporting sequences of treatments is an interesting avenue for future work. Focusing on a time-step $t$, this implies that $Y = Y_t$ in Equation \ref{eq:data-model-binary-y} and $T = T_t$ in Equation \ref{eq:data-model-binary-t}.

Defining the relevant feature set ($W$ in Eq. \ref{eq:data-model-binary-y}, \ref{eq:data-model-binary-t}) is more tricky. Recall from \S\ref{sec:causal-forecasting} that our time series consist in static ($S$) and temporal ($X_{1:t}$) features. These features are naturally included ($W \supset (S, X_{1:t})$).
In time-series, results from past time steps, such as past prices and associated demands $(T_{1:t-1}, Y_{1:t-1})$, are typically very predictive.
However, both the uncondoundedness Assumption \ref{eq:assumption-uncounfoundedness} and the orthogonal leaning data model (Eq. \ref{eq:data-model-binary-y}, \ref{eq:data-model-binary-t}) implicitly assume that features $W$ follow the observational distribution at training and prediction time.
This assumption is not verified when models are used to influence past decisions, and $T_{<t}, Y_{<t}$ cannot be included in $X$.
This is an important point: observational past treatment are very predictive of future demand, but confound the daily treatment effect and are the cause of reversed elasticity predictions we observed (\S\ref{sec:causal-background}, Figure \ref{fig:motivation}).
To avoid this confounding, we introduce a cut-off $\tau$ after which treatments can deviate from the distribution. That is, in time-steps $1:\tau$ the model is not used to affect treatments, while for $t>\tau$, the treatment $T_t$ can differ from the training distribution. We then set the features as $W = (S,X_{1:t}, T_{1:\tau}, Y_{1:\tau}) \triangleq W_t$, which we use as input for our daily treatment effect of $T_t$ on $Y_t$.

{\bf High dimensional effects.} The other extension we require is support for higher-dimensional treatments. We consider both categorical and linear treatments and treatment effects. We start with categorical treatments, for which we propose two different encodings. In our demand forecasting dataset, possible prices are discrete, following a large but tractable set of possible values (about thirty possible prices) of size $d \triangleq |\mathcal{T}|$. We extend the approach of \citet{robinson, Nie_Wager,foster2023orthogonal} (Equations \ref{eq:data-model-binary-y}, \ref{eq:data-model-binary-t}) to model each time-step $t$ as follows:
\begin{align}
    & Y_t = T_t^T\theta_0(W_t) + f_0(W_t) + \epsilon_1, & \mathbb{E}[\epsilon_1|W_t,T_t] = 0  \label{eq:data-model-categ-y}\\
    & T_t = e_0(W_t) + \epsilon_2, & \mathbb{E}[\epsilon_2|W_t] = {\bf 0}_d \label{eq:data-model-categ-t}
\end{align}
where $T_t$ is a $d$-dimensional column vector encoding the treatment, $\theta_0(W_t) \in \sR^d$ outputs a column vector treatment effect, and ${\bf 0}_d$ is the $d$-dimensional zero vector.
The R-loss becomes:
\begin{equation}\label{eq:r-loss-categ}
    l(\theta,{m,e},z_t) = \Big(Y_t - m(W_t) - \big(T_t - e(W_t)\big)^T\theta(W_t) \Big)^2 .
\end{equation}

We consider two encodings for the categorical treatment $T_t$ and treatment effect $\theta(W_t)$. In the {\em one-hot encoding}, $T_t$ is a one-hot vector of the treatment. As a result, the $i^{th}$ dimension of the treatment effect model encodes the CATE compared to $f_0$. That is, $\theta(W_t)_i = \E\big[Y_t(T_i) - f_0(W_t) \ | \ W_t \big]$.

In the {\em cumulative encoding}, the treatment model encodes the treatment effect of incremental price changes, such that the cumulative predictions encode the CATE: $\sum_{j=1}^i \theta(W_t)_j = \E\big[Y_t(T_i) - f_0(W_t) \ | \ W_t \big]$. In this case, $T_t$ is a vector of ones in dimensions $1:i$, and zeros after. While we do not do it in this work, the cumulative encoding can be useful to encode constraints, such as fixing the sign of incremental effects of price changes.

Given true categorical treatments and our flexible deep-learning models, we can reasonably assume that $\theta_0 \in \Theta$, and thus $\theta^* = \theta_0$. With enough treatment variation, our model will converge to the true CATE.
In practice however, demand forecasting datasets only see local variation around a typical price for any particular value of features $W_t$. As a result, a small risk under the R-loss
only ensures an accurate $\theta$ around likely prices. The model is under-constrained further from typical prices, where it will not learn relevant treatment effects.
This is a fundamental limitation of any causal approach, as local variations around a typical price implies a lack of positivity for most dimensions of $\theta(W_t)$.

In such cases, adding structure to the treatment effect model is practically useful. To this end, we also consider a {\em linear encoding} for the treatment effect. Concretely, we use the data model of Equations \ref{eq:data-model-categ-y} and \ref{eq:data-model-categ-t}, where $T$ is a one-dimensional, real treatment value (the price at which we want to predict demand), and the treatment effect $\theta(W_t)$ is a multiplicative coefficient. This creates a linear effect model, where the linear coefficient depends on $W_t$ in a flexible way. As a result, the real treatment effect is probably not in the model class ($\theta_0 \not\in \Theta$), and we converge to $\theta^*$. In exchange, estimations under this model can be much more stable under moderate treatment variability.

All encodings are compatible with the following result:
\begin{proposition}[Orthogonal Learning for Categorical Treatment Effects]\label{prop:orthogonal-learning-categ}
Consider a class of functions $\Theta: \mathcal{W} \rightarrow \sR^d$, where $\forall W, \ \forall \theta \in \Theta, \ \forall i \in \{1,d\}: \ |\theta(W)_i| \leq M$.
Fitting $g$ and $\theta$ with the procedures above using Algorithm \ref{alg:orthogonal-learning}, we have that with probability at least $1-\delta$:
\begin{align*}
\E\big[ \big(T - e_0(W)\big)^T\big(\theta(W) - \theta^*(W) \big) \big]^2 & \leq L_\cD(\theta, g_0) - L_\cD(\theta^*, g_0) \\
     & \leq \textrm{Rate}_\cD(S_2, \delta/2; \theta, g) + 2(1+d M^2) \cdot \textrm{Rate}_{\cD}(S_1, \delta/2)^2 .
\end{align*}
\end{proposition}
\begin{proof} Under the data model of Equations \ref{eq:data-model-categ-y}, \ref{eq:data-model-categ-t}, the categorical R-loss from Equation \ref{eq:r-loss-categ} satisfies universal Neyman orthogonality (Assumption \ref{eq:Universal_orthogonality}, proof in Appendix \ref{sec:appendix_assumptions_1}), and has continuous and bounded second directional derivatives (Assumption \ref{eq:Boundness}, proof in Appendix \ref{sec:appendix_assumptions_2}) with $\beta=2(1+d M^2)$. Applying Theorem 2 from \citet{foster2023orthogonal} concludes the proof.
\end{proof}
Since for the linear encoding (and for binary treatments from Proposition \ref{prop:orthogonal-learning-binaray}) $d=1$, the rate of convergence for nuisance parameters is less impactful than for more flexible categorical encodings.

\subsection{Causal Time Series Forecasting Models with Deep Learning}
\label{sec:causal-forecasting-models}

To instantiate the theory described in \S\ref{sec:causal-forecasting-theory}, we extend time-series architectures for deep learning, that we use as backbones for models $m, e, \theta$. We then adapt the training procedure to fit Algorithm \ref{alg:orthogonal-learning}. Finally, we change the prediction procedure to output causal forecasts.

{\bf Backbone Time-Series Model Architecture with Orthogonal Learning.}
Our time-series backbone takes as input static features $S$, temporal features $X_{1:t}$, and treatments and outcome values before $\tau$, $T_{1:\tau}$ and $Y_{1:\tau}$. It is then used in three models: $e(\cdot)$, which predicts the treatment sequence $T_{>\tau}$; $m(\cdot)$, which predicts expected outcomes (without knowing the treatments); and $\theta(\cdot)$, which predicts a vector of treatment effects (interpreted differently depending on the encoding, see \S\ref{sec:causal-forecasting-theory}). Our inputs are non-traditional, and the reason why our main backbone is a modified TFT \cite{TFT} (modification details in Appendix \ref{sec:appendix_tft}). We also experiment with the state-of-the-art iTransformer architecture \cite{itransformer}, passing static features' embeddings along temporal ones, and training with randomly sampled $t$ values by truncating the time-series at $t$, as the iTransformer processes entire time-series without a causal structure (i.e., the model ``knows the future''). This results in very slow training and prediction, and we can only apply it to one of our datasets.

{\bf Fitting a Causal Orthogonal Time-Series Model.}
We fit our models following Algorithm \ref{alg:orthogonal-learning}. We train both models $m(W_t)$, $e(W_t)$ as estimators for $\mathbb{E}[Y_t|W_t]$ and $\mathbb{E}[T_t|W_t]$ on subset $S_1$. $m(\cdot)$ is trained with the mean squared error (MSE) loss. The loss for $e(\cdot)$ depends on the encoding function. We use the cross entropy for the one-hot encoding case, the binary cross-entropy for the cumulative encoding, and the MSE for the linear encoding. In all cases, we train $\theta(\cdot)$ by minimizing the R-loss $l(\cdot,m,e,z_t)$ from Eq. \ref{eq:r-loss-categ} with data points $z_t \in S_2$.  

{\bf Forecasting with a Causal Orthogonal Time-Series Model.}
$\theta$ predicts causal changes of outcomes under different treatments.
However, to optimize downstream decisions (e.g. chose prices to maximize revenue) we need a proper forecast for the outcome $Y_t$.
Our final estimator combines all models following Eq. \ref{eq:data-model-categ-y}:
\vspace{-\baselineskip}
\[
\hat{Y}_t(W_t) = m(W_t) + \big(T_t - e(W_t)\big)^T \theta(W_t) .
\]
This way under $g_0$ we have $\hat{Y}_t(W_t) = f_0(W_t) + e_0(W_t)^T \theta(W_t) + \big(T_t - e_0(W_t)\big)^T \theta(W_t) = f_0(W_t) + T_t^T \theta(W_t)$. When $\theta(\cdot)$ is close to the CATE, our predictions change causally with $T_t$.

\section{Regression Discontinuity Design for Causal Model Evaluation}
\label{sec:rdd}
Once our causal forecasting models are fitted, how can we measure their performance? This is a core challenge in causal models, which fundamentally rely on untestable assumptions \cite{pearl2009causality}, such as Assumption \ref{eq:assumption-uncounfoundedness} or related data models as in Equations \ref{eq:data-model-binary-y}-\ref{eq:data-model-categ-t}.
In \S\ref{sec:causal-background}, we used domain knowledge on the sign of the CATE to argue that non-causal models fail to learn valid causal relationships. However, these are qualitative observations that do not enable comparisons between different models.
To the best of our knowledge, their is currently no reliable method to quantitatively evaluate the quality of causal models, without relying on the same assumptions used in fitting the models.

To address this challenge, we take inspiration from robustness checks common in econometrics, which estimate causal effects of interest using multiple approaches. We scale this approach to evaluate ML models. Specifically, we leverage Regression Discontinuity Designs (RDDs) (\S\ref{sec:rdd-background}) and changes in treatments over time to estimate a subset of daily CATE values in a time-series. This approach uses a different assumption than Assumption \ref{eq:assumption-uncounfoundedness} (\S\ref{sec:rdd-estimate-cate}). We then compute CATE values for our test set time-series, and use the resulting causal test to evaluate and compare ML models (\S\ref{sec:rdd-dataset}).

\subsection{Background: Regression Discontinuity Designs}
\label{sec:rdd-background}

Regression Discontinuity Designs (RDD) leverage a continuity assumption to estimate a treatment effect at a cut-off point that triggers a change in treatment \cite{hahn2001identification,imbens2008regression,lee2010regression}.
We present a slightly non-traditional version that conditions on additional variables, needed for our demand forecast setting.
Formally, RDDs require a variable X with an associated cut-off value $c$ that corresponds to a change in treatment. That is, $T=T_1$ when $X < c$, and $T=T_2$ when $X \geq c$. RDDs provide identifyability under the following assumption:
\begin{assumption}[Continuity]\label{eq:assumption-rdd-continuity} The potential outcomes' conditional expectation, $\E[Y(T_2)| X=x, V]$ and $\E[Y(T_1)| X=x, V]$ respectively, are both continuous in $x$.
\end{assumption}

Under assumption~\ref{eq:assumption-rdd-continuity}, we can estimate the CATE at the cutoff $X=c$:
\begin{proposition}[CATE Identifyablity from RDD]\label{prop:rdd-base}
Under Assumption~\ref{eq:assumption-rdd-continuity}, we have that:
\[
\E\big[Y(T_2) - Y(T_1) \mid X = c, V \big] = \lim_{x \to c^+} \E\big[ Y \mid X = x, V\big] - \lim_{x \to c^-} \E\big[ Y \mid X = x, V\big]
\]
\end{proposition}

Practical estimators based on Proposition \ref{prop:rdd-base} typically fit linear or polynomial models $g: \mathcal{X}, \mathcal{V} \rightarrow \mathcal{Y}$, using an indicator variable for the cut-off to measure the discontinuity at the point of change of treatment $X=c$. This discontinuity captures the CATE at $X=c$. Since we are interested in the limit at $X=c$, one often uses a weight kernel $K(\cdot)$ that decreases the importance of datapoints further from the cut-off $X=c$. Typical weight kernels include the rectangular "window" kernel of width $h$, or a triangular kernel in which weights decay linearly when moving away from $X=c$, until they reach zero outside of $[c-h, c+h]$.

\subsection{Estimating Treatment Effects in Time Series}
\label{sec:rdd-estimate-cate}

RDDs are attractive to evaluate our models, as Assumption \ref{eq:assumption-rdd-continuity} differs from uncounfoundedness (Assumption \ref{eq:assumption-uncounfoundedness}) which underpins the learning of causal models.
To causally evaluate models, we want to compare a model's CATE prediction on a given time-series $n$ at time-step $t$, to another estimate of this CATE used as ground truth.
We obtain this ground truth by framing observed changes in treatment on individual time-series as an RDD.
Consider time-series $n$. We call $t^n_i$ the $i^{th}$ {\em switching time-step} at which a price change happens, with $T_1 \triangleq T_{t^n_i - 1} \neq T_{t^n_i+1} \triangleq T_2$.
Because we use aggregated time-series (e.g., at the daily level), and treatment changes happen in the middle of an observation, $T_{t^n_i}$ is often ill defined. While this deviates from traditional RDD formulations, in which $\E\big[ Y \mid X = c, V\big] = \E\big[ Y(T_2) \mid X = c, V\big]$, Prop. \ref{prop:rdd-base} and associated estimators still apply.

Another deviation from \S\ref{sec:rdd-background} is that our time-steps are discrete, so we technically cannot take the upper- and lower-limits at the $t^n_i$. We follow \citet{lee2008regression}
and assume that specification errors at $t^n_i$ (the deviation from $g$'s estimates and the true conditional expectation on each side of the cut-off) are zero in expectation. Formally:
\begin{assumption}[Unbiased specification errors]\label{eq:assumption-rdd-continuity-with-model} Consider (continuous at $c$) RDD models $g_{T_1}: x \rightarrow g(T_1, x)$ and $g_{T_2}: x \rightarrow g(T_2, x)$ for time-series $n$.
For $i\in \{1, 2\}, \ \E[Y_{t^n_i}(T_i) - g_{T_i}(c)| V]=0$.
\end{assumption}
In demand forecasting, assumptions \ref{eq:assumption-rdd-continuity} and \ref{eq:assumption-rdd-continuity-with-model} with $V=\emptyset$ are unrealistic, as there are known cyclical patterns of demand based on days of the week. In \S\ref{sec:rdd-dataset}, we show how we use $V$ to correct for these effects, making Assumptions \ref{eq:assumption-rdd-continuity} and \ref{eq:assumption-rdd-continuity-with-model} more legitimate.
With these assumptions and estimator, and denoting $w^n_{t^n_i}$ the features we give our orthogonal forecasting model at the cut-off (see \S\ref{sec:causal-forecasting-theory}), and $g^n$ the associated RDD model, we have the following result:
\begin{proposition}[RDD for point CATE]\label{prop:rdd-time-series} Consider time-series $n$ and switching-time $t^n_i$:
\[
\E\big[Y_{t^n_i}(T_{t^n_{i+1}}) - Y_{t^n_i}(T_{t^n_i}) \ | W_{t^n_i}=w^n_{t^n_i}, V \big]
=
\lim_{t \to {t^n_{i}}^{+}} \mathbb{E}\big[ g^n(t) \mid ~V\big] - \lim_{t \to {t^n_{i}}^-} \mathbb{E}\big[ g^n(t) \mid ~V\big] .
\]
\end{proposition}
\begin{proof} We start by using the fact that an i.i.d. time-series with a given context $W_t$ is an unbiased (one point) estimate of the conditional expectation over the data, that is: $\E\big[Y_{t^n_i}(T_{t^n_{i+1}}) - Y_{t^n_i}(T_{t^n_i}) \ | W_{t^n_i}=w^n_{t^n_i},  \mid t=t^n_i  \big] = \E\big[Y^n_{t^n_i}(T_{t^n_{i+1}}) - Y^n_{t^n_i}(T_{t^n_i})\big] = \E\big[Y^n_t(T_{t^n_{i+1}}) - Y^n_t(T_{t^n_i}) \mid t=t^n_i \big]$.

Invoking Assumption \ref{eq:assumption-rdd-continuity} and Proposition \ref{prop:rdd-base} yields $\E\big[Y^n_t(T_{t^n_{i+1}}) - Y^n_t(T_{t^n_i}) \mid t=t^n_i \big]
=
\lim_{t \to {t^n_{i}}^{+}} \E\big[Y^n \mid t, V \big] - \lim_{t \to {t^n_{i}}^{-}} \E\big[Y^n \mid t, V \big]$.
Finally, Assumption \ref{eq:assumption-rdd-continuity-with-model} gives:
$\lim_{t \to {t^n_{i}}^{+}} \E\big[Y^n \mid t, V \big] - \lim_{t \to {t^n_{i}}^{-}} \E\big[Y^n \mid t, V \big]
=
\lim_{t \to {t^n_{i}}^{+}} \E\big[g^n(t) \mid V \big] - \lim_{t \to {t^n_{i}}^{-}} \E\big[g^n(t) \mid  V \big]$.
\end{proof}

\subsection{Causal Test Sets to Evaluate Causal Forecasting Models}
\label{sec:rdd-dataset}

Proposition \ref{prop:rdd-time-series} lets us estimate the CATE on a specific time-series $n$, at a specific time $t^n_i$, for a specific treatment change $T_{t^n_i - 1} \rightarrow T_{t^n_i+1}$. We call this quantity $\text{CATE}_{t^n_i}$. To create our causal test set, we take the in-distribution task test set $S_{\textrm{test}}$, and collect all CATE values that we can estimate from it in $S^{\textrm{CATE}}_{\textrm{test}}$.
We only include switch-times $t^n_i$ that have at least three time-steps without treatment changes strictly before and after $t^n_i$, to have enough data for the RDD estimator.

Given enough data, we create a small dataset $D_i^n$ to fit our RDD model $g$, including all time-steps with constant treatments around $t^n_i$. Formally, $D_i^n = \{t: \forall t' \in [t, t^n_i-1], \ T_{t'} = T_{t^n_i-1} \} \cup \{t: \forall t' \in [t^n_i+1, t], \ T_{t'} = T_{t^n_i+1} \}$ (note that $t^n_i \not\in D_i^n$).
In tasks with known cyclical patterns, such as our demand prediction task with day of the week effects (see \S\ref{sec:rdd-estimate-cate}), we create a set of comparable time-series $S^n$ (using a rule informed by domain knowledge) on which we fit a linear regression model with time interactions $Y \sim V + V \times t$. For days of the week, this gives us parameters $\alpha_{1,j}, \alpha_{2, j}$ for $j\in \{1, \ldots, 7\}$. We then compute the residuals $\Tilde{Y}_t = Y_t - \sum_{j=1}^7(\alpha_{1,j} + t \alpha_{2,j})$, on which we fit our RDD model.
From early experiments, we noticed that $L_2$ regularized linear models perform best for $g$.
We fit such a model with weight kernel $K(.)$ of window size $h$ on dataset $D_i^n$ using the following specification:
\[
\tilde{Y}_t = \beta_0 + \beta_1 (t-t^n_i) + \beta_2 \1_{t > t^n_i} + \beta_3 \1_{t > t^n_i} (t-t^n_i) + \epsilon
\]
Parameter $\beta_2$ corresponds to the estimate of $\text{CATE}_{t^n_i}$ that we put in our causal test set $S^{\textrm{CATE}}_{\textrm{test}}$.
Figure \ref{fig:rdd} represents the estimation of different $\text{CATE}_{t^n_i}$ for a given time series, with and without correction.

\begin{figure}[t]
\centering
\begin{subfigure}{.49\textwidth}
  \centering
  \includegraphics[scale=0.29]{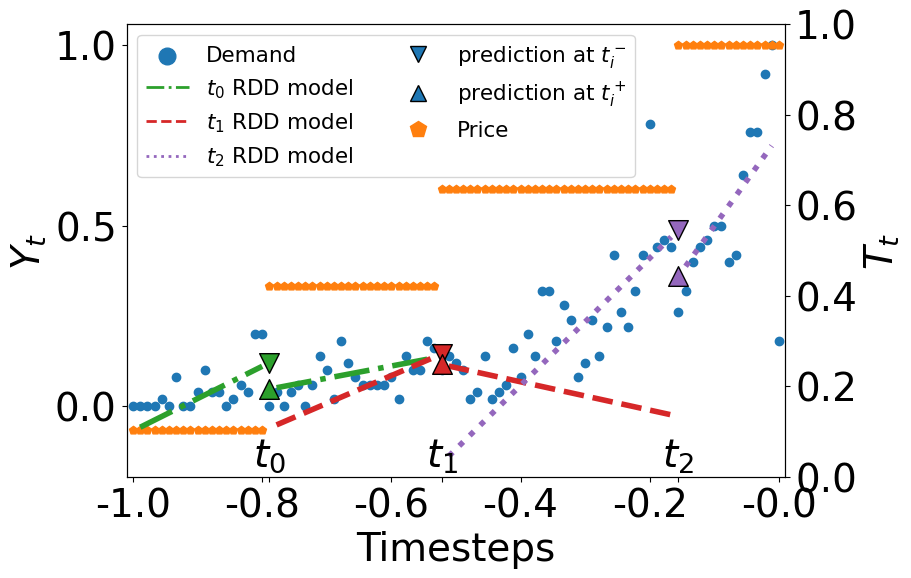}
  \caption{RDDs without weekday correction}
  \label{fig:rdd_linear}
\end{subfigure}%
\hfill
\begin{subfigure}{.49\textwidth}
  \centering
  \includegraphics[scale=0.29]{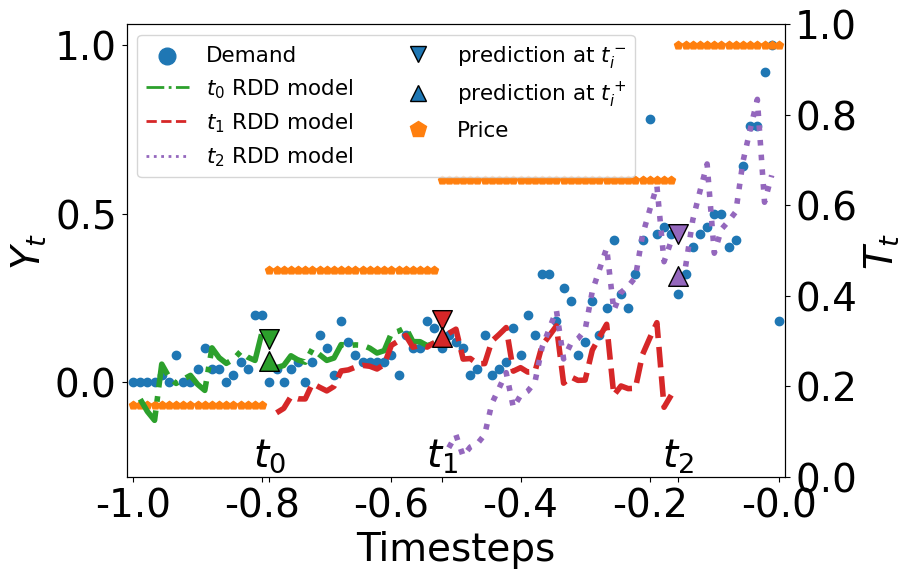}
  \caption{RDDs using weekday correction}
  \label{fig:rdd_weekday}
\end{subfigure}
\vspace{-0.5em}
\caption{CATE estimates (distance between the two triangles) at three $t^n_i$ values on a train time series $n$, without weekday correction (a), and with correction (b). In this example, the RDD framework effectively captures the decline in demand under price increases. The fit is better after correction, so CATE estimates are likely more accurate. Prices and demand are scaled to $[0, 1]$. }
\vspace{-1em}
\label{fig:rdd}
\end{figure}

Small sample sizes in RDD datasets make the final causal test set $S^{\textrm{CATE}}_{\textrm{test}}$ noisy. We filter outliers by cutting the 2.5\% tails of the distribution on each side. With this test set of causal effect, we can evaluate predictions of causal effects from our models, as described in \S\ref{sec:causal-forecasting-models}. We use traditional metrics such as the root-mean-square error (RMSE) or the mean absolute error (MAE). We can compare such metrics to causal effects estimated with non-causal models, by predicting with those models at different treatments (prices) and subtracting the results to estimate their predicted CATE.

\section{Evaluation}
\label{sec:eval}

\subsection{Demand Forecasting Dataset}
\label{sec:eval-wm}
{\bf Dataset.} Each time-series is a sequence of prices and number of sales for one train. Temporal features include the weekday of the timestep and the number of days before departure. Stationary features include the departure and arrival terminals and departure date time. The dataset consists in about 300,000 training time-series, and validation and test sets of around 100,000 each. All the figures and metrics regarding this dataset have been anonymized. Time steps $t$ are normalized to $[-1,0]$ (all time series have the same length).

\textbf{Experimental setup.} We set $\tau=0.33$, and forecast all outcomes $(Y_t)_{t>\tau}$. Baseline (non-causal) models are trained with the MSE loss, and include a well-tuned LGBM model (used in production) and two deep learning architectures: the state-of-the-art iTransformer \cite{itransformer}, and a custom version of the Temporal Fusion Transformer (TFT) \cite{TFT}. The TFT maps well to our requirements (see \S\ref{sec:causal-forecasting}), and we modify it for performance (details in Appendix \ref{sec:appendix_tft}).
We tune all models' hyper-parameters separately and equally, and report metrics by training and testing with 5 random seeds on final hyper-parameters, and reporting the mean and standard-deviation.
The RDD dataset uses a linear weights kernel with $h = 14$ time steps.
Keeping all switching times with at least 3 data points strictly before and after the price change retains 70\% of observed price changes.

\textbf{Results.} Table \ref{tab:WM_table} shows the performance of several models. We can see that non-causal models perform best in-distribution, and the LGBM is best among them in both RMSE and MAE. This is expected, as the observed prices are very informative in-distribution, since operators use their experience to set prices based on the demand they expect. On causal tasks however, causal models outperform the baselines, with linear models being particularly accurate in terms of RDD RMSE, which is expected in high teatment dimension $d$ (\S\ref{sec:causal-forecasting-theory}).
The Causal iTransformer with linear activation has an RDD RMSE 37\% better than the baseline TFT, the next best non-causal model, a significant improvement.
Figure \ref{fig:eval_WM_train} corresponds to the same time series as figure \ref{fig:motivation-train-example-LGBM}-\ref{fig:motivation-train-example-TFT}, and shows that the causal model captures the correct CATE sign (increasing prices decreases the demand forecast).
Figure \ref{fig:eval_WM_histogram} shows that the causal TFT's CATE distribution is qualitatively much better: model predicts fewer positive CATE effects than baselines on the test set, and those positive CATEs are smaller in magnitude.

\begin{table}[t]
\begin{center}
\vspace{0.5em}
\resizebox{14cm}{!}{
\begin{tabular}{lllll}

\multicolumn{1}{c}{\bf Models}  &\multicolumn{1}{c}{\bf RMSE} &\multicolumn{1}{c}{\bf MAE} &\multicolumn{1}{c}{\bf RDD RMSE } 
&\multicolumn{1}{c}{\bf RDD MAE }
\\ \hline 
LGBM         &\bf 0.6870 $\pm$ 0.0015  &\bf 0.2877 $\pm$ 0.0005 & 0.2163 $\pm$ 0.0009 & 0.1390 $\pm$ 0.0004\\
TFT Baseline         &0.7044 $\pm$ 0.0050  & 0.3023  $\pm$ 0.0053 & 0.1910 $\pm$ 0.0014 & 0.1262 $\pm$ 0.0005\\
\hline
Causal TFT Linear         &0.7216 $\pm$ 0.0020  & 0.3104 $\pm$ 0.0047 & {\bf 0.1217 $\pm$ 0.0019} & 0.1271 $\pm$ 0.0008\\
Causal TFT Cumulative       &0.7189 $\pm$ 0.0091  & 0.3067 $\pm$ 0.0019 & 0.1820 $\pm$ 0.0008 & \bf 0.1243 $\pm$ 0.0004\\
Causal TFT One-hot        &0.7236 $\pm$ 0.0115  & 0.3095 $\pm$ 0.0069 & 0.1874 $\pm$ 0.0007 & 0.1279 $\pm$ 0.0002\\
\hline
Causal iTransformer Linear         &0.7504 $\pm$ 0.0098  & 0.3113 $\pm$ 0.0038 & \bf 0.1197 $\pm$ 0.0014 & 0.1259 $\pm$ 0.0003\\
Causal iTransformer Cumulative         &0.7377 $\pm$ 0.0070  & 0.3058 $\pm$ 0.0017 & 0.1908 $\pm$ 0.0021 & 0.1282 $\pm$ 0.0009\\
Causal iTransformer One-hot             &0.7347 $\pm$ 0.0097  & 0.3036 $\pm$ 0.0030 & 0.1930 $\pm$ 0.0023 & 0.1288 $\pm$ 0.0009\\
\end{tabular}
}
\caption{Baselines perform better in distribution (RMSE, MAE), while causal models better capture causal effects (RDD metrics).}
\label{tab:WM_table}
\end{center}
\end{table}

\begin{figure}[t]
\centering
\begin{subfigure}{.4\textwidth}
  \centering
  \includegraphics[scale=0.375]{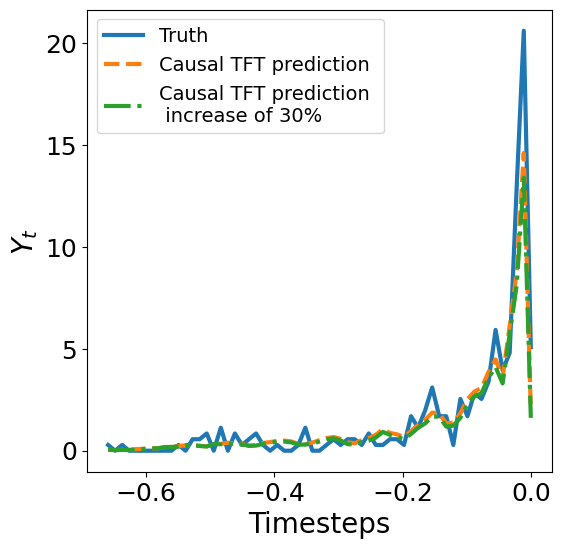}
  \caption{Causal TFT}
  \label{fig:eval_WM_train}
\end{subfigure}%
\hfill
\begin{subfigure}{.4\textwidth}
  \centering
  \includegraphics[scale=0.375]{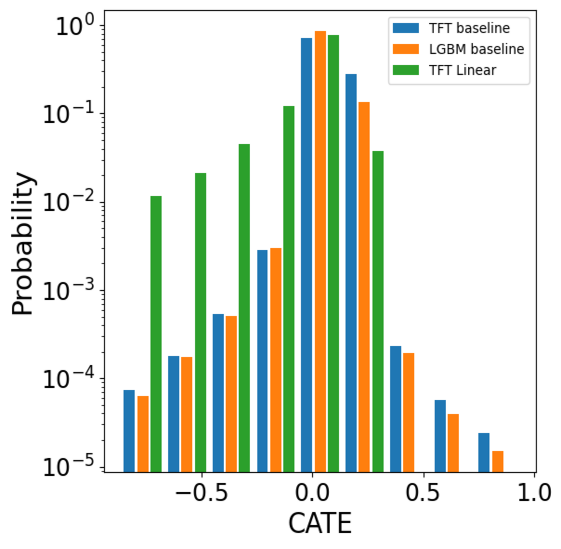}
  \caption{Treatment effects distribution}
  \label{fig:eval_WM_histogram}
\end{subfigure}
\vspace{-0.5em}
\caption{(a) Daily demand forecast change under price increase for one train, causal TFT. Analogous to Fig. \ref{fig:motivation-train-example-LGBM}-\ref{fig:motivation-train-example-TFT}. (b): CATE distribution over the test set, for a change from the observed price to the next possible price, normalized by the price change. Analogous to Fig. \ref{fig:motivation-elasticity}.}
\label{fig:Causal_WM_train}
\vspace{-1em}
\end{figure}

\subsection{MIMIC Health Prediction Dataset}
\label{sec:eval-mimic}
We also evaluate our approach on a public dataset, and release the code to replicate experiments here: \url{https://github.com/wiremind/causal-forecasting}.

\textbf{Dataset.} We use MIMIC-extract library \cite{MIMIC-extract} to process MIMIC-III \cite{MIMIC-III}.
Every time-series represents data from one patient who stayed between 30 and 60 hours in a critical care unit, a typical setting in causal inference benchmarks \cite{adversarial_loss,causal_forecasting,causal_transformers}. Static features include the patient's gender, age, and ethnicity. Temporal features include 25 vital signals.
Our task is to forecast a patient's blood pressure, and we aim at estimating the treatment effect of mechanical ventilation and vasopressor drugs.
Since this represents a categorical treatment, our linear encoding does not apply.

\textbf{Experimental setup.} We keep the default variable length time-series, and follow prior art in forecasting blood pressure for the next 5 hours \citep{causal_forecasting,causal_transformers}. That is, for each $t$ we set $\tau = t-1$ and forecast $[t, t+4]$.
We use the same baseline losses as in \S\ref{sec:eval-wm}.
We were not able to fit our causal models with the iTransformer backbone, which requires a fixed $\tau$, so we focus on our causal TFT.
We also consider a state-of-the-art causal forecasting baseline, the Causal Transformer \cite{causal_transformers}. We slightly modify the code to plug information leakage from the future vitals of hour $t=\tau+1$ (see Appendix \ref{sec:appendix_causal_transformer} for details).
We evaluate forecasts in distribution as prior work, as well as using our RDD methodology to create a causal test set. Our RDDs use a rectangular weight kernel with $h = 5$.
Retaining switching times with at least three points strictly on each side also preserves about 70\% of the treatment changes.

\textbf{Results.} We focus on RMSE results. MAE results are in Appendix \ref{appendix:mimic} and are qualitatively the same.
When forecasting in distribution, Table \ref{tab:MIMIC_table_RMSE} (up) shows that baseline models outperform causal ones.
The auto-regressive Causal Transformer outperforms our TFT in next-step prediction, but is less accurate when predicting further into the future, which is consistant with prior observations \cite{itransformer,Tide,autoformer,SAMFormer}. 
Once again, baseline models are able to exploit the correlation between treatments and future outcomes, as doctors use their domain knowledge to anticipate ant prevent issues.

On the causal task using our RDD-based CATE estimates though, we can see on Table \ref{tab:MIMIC_table_RMSE} (down) that the one-hot causal model outperforms the baseline TFT by at least 1.6\% in all time-steps, and the Causal Transformer by almost 1\% in the first time step.
Contrary to our Causal TFT, the Causal Transformer has access to the treatment sequence between $\tau$ and $t$. Since our RDD methodology only estimates daily treatment effects, the sequence of treatments $\tau:t$ is always in distribution, and the Causal Transformer is able to exploit this signal. However, this means that Causal Transformer CATE predictions for changes in {\em treatment sequences} would be confounded, and suffer some the same issues as the TFT baseline.
Finding similar estimators to our RDD approach for treatment sequences to rigorously evaluate these effects is an interesting avenue for future work.

\begin{table}[t]
\begin{center}
\vspace{0.5em}
\resizebox{14cm}{!}{
\begin{tabular}{llllll}
\multicolumn{1}{c}{\bf Models}  &\multicolumn{1}{c}{ $t$ = $\tau$+1} &\multicolumn{1}{c}{ $t$ = $\tau$+2} &\multicolumn{1}{c}{ $t$ = $\tau$+3 } &\multicolumn{1}{c}{ $t$ = $\tau$+4 } &\multicolumn{1}{c}{ $t$ = $\tau$+5 }
\\ \hline
\multicolumn{6}{c}{In distribution RMSE}
\\ \hline
Causal Transformer         &\bf 8.823 $\pm$ 0.028  & \bf 9.499 $\pm$ 0.024 & 9.845 $\pm$ 0.035 & 10.098 $\pm$ 0.028 & 10.313 $\pm$ 0.026\\
TFT Baseline         &8.921 $\pm$ 0.056  & 9.534 $\pm$ 0.057 & \bf 9.834 $\pm$ 0.061 & \bf 10.051 $\pm$ 0.052 & \bf 10.238 $\pm$ 0.043\\
Causal TFT One-hot        &8.951 $\pm$ 0.023  & 9.596 $\pm$ 0.025 & 9.921 $\pm$ 0.038 & 10.165 $\pm$ 0.028 & 10.371 $\pm$ 0.027\\
Causal TFT Cumulative        &8.945 $\pm$ 0.026  & 9.608 $\pm$ 0.033 & 9.935 $\pm$ 0.046 & 10.184 $\pm$ 0.035 & 10.390 $\pm$ 0.030
\\ \hline
\multicolumn{6}{c}{RDD RMSE for CATE prediction}
\\ \hline
Causal Transformer         & 2.887 $\pm$  0.145  & 2.869 $\pm$ 0.145 & \bf 2.844 $\pm$ 0.141& \bf 2.833 $\pm$ 0.140 & \bf 2.833 $\pm$ 0.129\\
TFT Baseline         &2.908 $\pm$ 0.102 &2.942 $\pm$ 0.103 & 2.969 $\pm$ 0.114& 3.011 $\pm$ 0.104 & 3.051 $\pm$ 0.096\\
Causal TFT One-hot        &  \bf 2.861 $\pm$ 0.094  & \bf 2.861 $\pm$ 0.094 & 2.861 $\pm$ 0.094 & 2.862 $\pm$ 0.094 & 2.862 $\pm$ 0.094\\
Causal TFT Cumulative        & 2.895 $\pm$ 0.050  & 2.895 $\pm$ 0.050 & 2.895 $\pm$ 0.050& 2.895 $\pm$ 0.050 & 2.895 $\pm$ 0.050\\
\end{tabular}
}
\caption{Performance of next five steps forecasts, in distribution and for causal effects (RDD).}
\label{tab:MIMIC_table_RMSE}
\end{center}
\vspace{-1em}
\end{table}


\bibliography{iclr2025_conference}
\bibliographystyle{iclr2025_conference}

\appendix
\section{Encoding functions}
\label{sec:encoding_functions}
In this section, we expand on the three different encoding functions that we plug into the R-Loss.

\subsection{Linear encoding}\label{sec:scalar-encoding}

A possible encoding function is to pass the treatment as a scalar value in the R-Loss, giving a linear approximation of $\E[Y_t(T) | W_t]$ with regards to $T$. In this framework, we denote $T_t$ as the scalar value of the treatment. Hence the data model reduces to that of Eq. \ref{eq:data-model-binary-y}, but with a continuous $T$.

We can access the CATE between $T$ and $T'$ as:
\begin{equation*}
    \mathbb{E}[Y(T_t=T,W_t) - Y(T_t=T',W_t)|W_t] = \theta_0(W_t)(T-T')
\end{equation*}
This encoding performs best on the the railway dataset, in which treatments are fairly high dimensional and, conditioned on $W_t$, we only observe local variation of the treatment around a typical price.
Figure \ref{fig:theta_reg} is an example of the estimated values of $\theta_0$ for a random time-series. We notice that the effect of prices has little impact at the beginning of the forecast. The effect of price over sales increases as the departure date gets closer. Moreover, we observe that the price has little effect at time step $t=0$ (the day at which the train leaves leaves).

\begin{figure}[ht]
    \centering
    \includegraphics[width=0.5\linewidth]{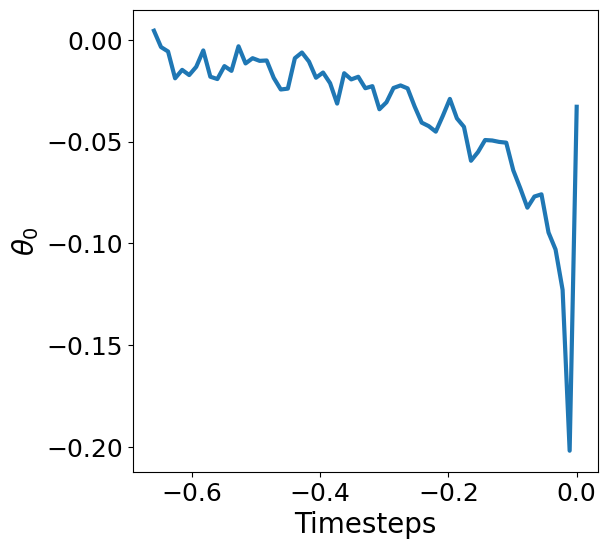}
    \caption{$\theta_0$ values for various time steps for a railway time serie}
    \label{fig:theta_reg}
\end{figure}

\subsection{One-hot encoding}
A more intuitive encoding under our categorical treatments is the one-hot encoding. In this case $e_0(W_t)_i = P(T_t=i|W_t)$. With $d$ possible treatments, the vector $\theta_0(W_t)$ is:
\begin{equation}\label{eq:meaning_theta_one_hot}
    \forall i \in [1:k], \theta_0(W_t)_i = \mathbb{E}[Y_t(T_t=i,W_t) - f_0(W_t)|W_t]
\end{equation}
In order to estimate the CATE between $T$ and $T$, we compute the difference between $\theta_0(W_t)_{T}$ and $\theta_0(W_t)_{T'}$:
\begin{equation*}
    \mathbb{E}[Y(T_t=T,W_t) - Y(T_t=T',W_t)|W_t] = \theta_0(W_t)_{T} - \theta_0(W_t)_{T'}
\end{equation*}

Thanks to equation \ref{eq:meaning_theta_one_hot} and to domain knowledge on the railway dataset, we can make the assumption that the values of $\theta_0(W_t)_i$ must decrease when $i$ increases. Figure \ref{fig:theta_one_hot} shows an example of $\theta$'s vector values for a specific time-series and time step. We notice that the values seem to decrease as the treatment increases. Nonetheless, we observe an increase for small treatment values. It is likely caused by the fact that those treatments are extremely rare in our dataset (conditioned on $W_t$).
\begin{figure}[ht]
    \centering
    \includegraphics[width=0.5\linewidth]{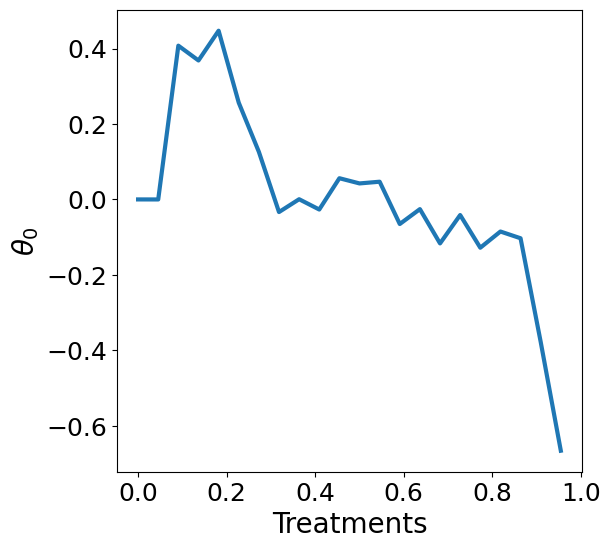}
    \caption{Example of $\theta$ vector values for the one-hot encoding on the passenger rail dataset.}
    \label{fig:theta_one_hot}
\end{figure}

\subsection{Cumulative encoding}

\begin{figure}[ht]
    \centering
    \includegraphics[width=0.5\linewidth]{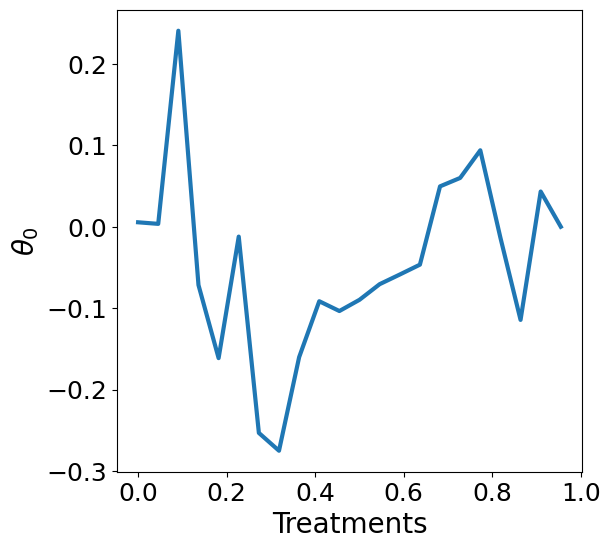}
    \caption{Example of $\theta_0$ vector values for the cumulative encoding }
    \label{fig:theta_cumul}
\end{figure}

The cumulative encoding function $x\mapsto \mathbb{1}_{T_t\ge x}$ is another possible encoding function of the treatment in our data model in Eq. \ref{eq:data-model-categ-y}. The interpretation of the vector $\theta_0(W_t)$ is then:
\begin{equation*}\label{eq:meaning_theta_cumul}
    \forall i \in [2:k], \theta_0(W_t)_i = \mathbb{E}[Y_t(T_t=i,W_t) - Y_t(T_t=i-1,W_t)|W_t] 
\end{equation*}
Thus the CATE between treatment $T$ and $T'$, with $T'<T$ equals:
\begin{equation*}
    \mathbb{E}[Y(T_t=T,W_t) - Y(T_t=T',W_t)|W_t] = \sum_{i=T'+1}^{T} \theta_0(W_t)_i
\end{equation*}
As done in the last subsection, we also can infer some characteristics of the theta vector values thanks to domain knowledge and equation \ref{eq:meaning_theta_cumul}. In that case, the values of $\theta_0$ should be negative for all indexes. Figure \ref{fig:theta_cumul} shows an example of $\theta(W_t)$. Most of them are indeed negative. Nonetheless, some values associated to small treatment values, and a few very large treatment values, are still positive. Again, this is likely caused by those treatments being rarely seen during training (conditioned on $W_t$).

\section{Detailed Theory and Proofs}
\label{sec:appendix_theory}

\subsection{Stating and Verifying Orthogonal Learning Assumptions}
\label{sec:appendix_assumptions}
Proposition \ref{prop:orthogonal-learning-categ} relies on two assumptions that must be verified by the R-loss to ensure convergence through Theorem 2 in \citet{foster2023orthogonal}.
These two assumptions are \textbf{Universal Orthogonality} and \textbf{Boundness}.
Denote the observations $Z^n = (Y_{1:t}^n,T_{1:t}^n, X_{1:t}^n,S^n)$, sampled from an unknown distribution $\mathcal{D}_t$ depending on $t$. We denote $\theta^* \in \arg\min_{\theta \in \Theta} L_{\mathcal{D}_t}(\theta,m,e)$, and $star(\Theta,\theta)$ the star domain of $\Theta$ at point $\theta$.
Both assumptions require the definition of a \textit{Directional Derivative}:
\begin{definition}[Directional Derivative] Let $F : \mathcal{F} \rightarrow \mathbb{R}$ be a function from a vector space $\mathcal{F}$, and define the derivative operator at point $f$ with respect to vector $g$ as:
\[
D_fF(f)[g] = \frac{d}{dt}F(f+tg)_{|t=0}
\]
\end{definition}

With this definition, we can state the two assumptions required of the (population level) R-loss (we omit the indexing on $t$ in $\mathcal{D}_t$, as the same assumption must old in each $t$ but our definition is identical at each $t$):
\begin{assumption}[Universal Orthogonality] For all $\hat{\theta} \in star(\hat{\Theta},\theta^*) + star(\hat{\Theta}-\theta^*)$  :
\[
D_gD_\theta L_D(\hat{\theta}, m,e)[\theta - \theta ^*, g-g_0] = 0
\]
\label{eq:Universal_orthogonality}
\end{assumption}
\begin{assumption}[Boundness] $D_gD_gL_D$ and $D_\theta D_\theta D_g L_D$ are both continuous and there is a constant $\beta$ such as $\forall \theta \in star(\hat{\Theta},\theta^*)$ and $\forall \Tilde{g} \in star(\mathcal{G},g_0)$:
\[
|D_gD_g L_D(\theta, \Tilde{g})[g - g_0, g-g_0]| \leq \beta ||g-g_0||_{\mathcal{G}}^2
\]
\label{eq:Boundness}
\end{assumption}

\subsection{Universal Orthogonality}\label{sec:appendix_assumptions_1}
We denote $W_t=(S,X_{1:t},Y_{1:\tau})$. We omit the indices $t$ for $T_t$, $Y_t$ and $W_t$.  We note $g_0 = (m_0,e_0)$.
Next, we prove the orthogonality of the R-Loss.
\begin{align*}
    &D_\theta L_D(\Bar{\theta},g_0)[\theta - \theta^\star]  \\ 
    &=-2\mathbb{E}[(T-e_0(W)^T \big(\theta(W) - \theta^\star(W)\big) \times \big(Y-m_0(W)-(T-e_0(W))^T\Bar{\theta}(W)\big)]
\end{align*}

We first consider the directional derive with respect to $e$:
\begin{align*}
    &D_eD_\theta L_D(\Bar{\theta},g_0)[\theta - \theta^\star,e-e_0] = \\ 
    &2\mathbb{E}[\big(Y-m_0(W) -(T-e_0(W))^T \Bar{\theta}(W)\big) \times \big(e(W)-e_0(W)\big)^T \big(\theta(W) - \theta^\star(W)\big)] \\
    &- 2\mathbb{E}[\big(e(W)-e_0(W)\big)^T \Bar{\theta}(W) \times \big(T-e_0(W)\big)^T \big(\theta(W)-\theta^\star(W)\big)]
\end{align*}
We notice that in the first term of this derivative, we have $\forall w\in \mathcal{W}$:
\begin{gather*}
    \mathbb{E}[((Y-m_0(W) - (T-e_0(W))^T\Bar{\theta}(W)|W=w] = \mathbb{E}[\epsilon_1 +\epsilon_2^T \big(\theta_0(W) - \Bar{\theta}(W)\big)|W=w] = 0
\end{gather*}
Regarding the second term, when developing the dot products we also have that $\forall w\in \mathcal{W}$:
\begin{align*}
    &\mathbb{E}[(e(W)-e_0(W))^T \Bar{\theta}(W) \times (T-e_0(W))^T (\theta(W) - \theta^\star(W))|W=w] \\
    &= \sum_{i=1}^k \sum_{j=1}^k \mathbb{E}[\epsilon_{2j} |W=w] \times \Bar{\theta}(w)_j\times (e(w)-e_0(w))_i \times (\theta(w) - \theta_0(w))_j \\
    &= 0
\end{align*}

We now need to establish the orthogonality for the parameter $m$ of the population risk:
\begin{align*}
    &D_m D_\theta L_D(\Bar{\theta},g_0)[\theta - \theta^\star,m-m_0]  \\
    &= 2 \mathbb{E}[(T-e_0(W))^T (\theta(W)-\theta^\star(W)) \times (m(W)-m_0(W))] \\
    &= 2\mathbb{E}[\epsilon_2^T (\theta(W)-\theta^\star(W)) \times (m(W)-m_0(W))] \\
    &= 0
\end{align*}
with $E[\epsilon_2|Z] = 0_d$.

Hence, the population level version of the R-loss from Eq. \ref{eq:r-loss-categ} verifies assumption \ref{eq:Universal_orthogonality}.

\subsection{Boundness}\label{sec:appendix_assumptions_2}

The second directional derivatives of the population risk are continuous as it is the composition the square function and a linear function, both with continuous gradients. 
We now prove the boundedness of the second derivative of the population risk with regard to the nuisance parameters, for the vector case. 

Let us denote $l(\gamma, \zeta,z) = [Y- \gamma_1 - (T-\gamma_2)^T \zeta]^2$ with $\gamma_1 \in \mathbb{R}$ and $(\gamma_2,\zeta) \in (\mathbb{R^d})^2$
We have:

\begin{gather*}
\nabla_{\gamma \gamma}l(\gamma,\zeta,z) = 2
    \begin{bNiceArray}{c|ccc}[margin]
    1 & -\zeta_1 & \Cdots &  -\zeta_d \\
    \hline
    -\zeta_1 & \Block{3-3}{(\zeta_i\zeta_j)_{(i,j)\in [1,d]^2}} \\
    \Vdots & & & \\
    -\zeta_d & & & \\
\end{bNiceArray}
\end{gather*}
This matrix rank is 1 as every line $i$ is equal to the first line multiplied by $-\zeta_i$. Moreover the unique non-zero eigenvalue of this matrix is $2*(1+\sum_{i=1}^d \zeta_i^2)$. Its eigenvector is $(1,-\zeta_1,...,-\zeta_d)$.

Denote $M$ the bound such as $\forall W, \forall i, |\theta(W)_i|<M$. Then when $\zeta_i = \theta(W)_i$, we have that $2*(1+\sum_{i=1}^d \zeta_i^2) \leq 2(1+d M^2)$.
The directional derivative of the R-loss function in Eq. \ref{eq:r-loss-categ} at $g-g_0$ is the derivative of the composition of the functions \ref{eq:r-loss-categ} and $\Tilde{g}+t(g-g_0)$, with respect to $t$. Hence, when taking two derivatives with respect to $t$, the equation become:
\begin{align*}
    ||D_gD_g R(\theta,\Tilde{g},z)[g-g_0,g-g_0] || = || (g-g_0)^T \nabla_{\gamma \gamma}l(\gamma,\zeta,z) (g-g_0)||
\end{align*}
as the hessian of $t \mapsto t(g-g_0)$ equals 0.   

Using the fact that the largest eigenvalue of $\nabla_{\gamma \gamma}l(\gamma,\zeta,z)$ is less than $2(1+d M^2)$, we have $||D_gD_g R(\theta,\Tilde{g})[g-g_0,g-g_0] || < 2(1+d M^2) ||g-g_0||_{\mathcal{G}}^2$. Finally, we compute the expectation on this inequality and we apply the dominated convergence theorem to get the population risk (requiring that $||g-g_0||_{L_2} < \infty$).
To conclude, we get the following inequality:
\begin{gather*}
    ||D_gD_g L_D(\theta,\Tilde{g})[g-g_0,g-g_0] || < 2(1+d\times M^2)||g-g_0||_{L_2}^2
\end{gather*}
Thus, assumption \ref{eq:Boundness} is verified with $\beta = 2(1+d M^2)$.

\section{Implementation details}
\label{sec:implementation_details}
\subsection{Causal Transformer}
\label{sec:appendix_causal_transformer}
We use the Causal Transformer \citep{causal_transformers} in order to get a reference model for comparison. We use the public github repository (\url{https://github.com/Valentyn1997/CausalTransformer/tree/main}). \citet{causal_transformers} train an auto regressive model, like us using 25 vital signals, the age, the gender and the ethnicity to estimate the blood pressure of patient.
In that context, the relevant features for predicting blood pressure at time $t$ are all the vitals before time $t$, static features and all treatments up until time $t$. Nonetheless, when evaluating for time $\tau+l$ for $l$ in $[1,5]$, the model had knowledge of vitals at time $\tau+1$, enabling the model to perform very well for $l=1$ and leaking information for all $\tau$ (\url{https://github.com/Valentyn1997/CausalTransformer/blob/c49a24faa57af966501e241f57a26b528f874a53/src/data/mimic_iii/real_dataset.py#L67}). This leak manifested as extremely good performance for $l=1$, followed by a large drop of performances at $l \geq 2$. In our evaluation, we fix the leak from the future, and obtain results consistent with other models trained on this dataset during our study.
\subsection{RDD}
\label{sec:appendix_rdd}
Algorithm \ref{alg:rdd_dataset} describes how to compute the dataset of RDD values in full details.

\begin{algorithm}[H]
        \caption{RDD Algorithm}
    \begin{algorithmic}

    \State \textbf{Input:} $N$ time series $(Y_t, T_t, W_t)$
    \State
    \State Initialize $S^{\text{CATE}}_{\text{test}} \gets \{\}$

    \State

    \For{$n$ \textbf{from} $1$ \textbf{to} $N$}
        \State Initialize $k \gets 1$
        \For{$t$ \textbf{from} $1$ \textbf{to} $\text{length}(Y^n) - 1$}
            \If{$T_t \neq T_{t+1}$}
                \State $t^n_k \gets t$
                \State $k \gets k+1$
            \EndIf
        \EndFor

        \If{$k > 2$}

            \For{$i$ \textbf{from} $1$ \textbf{to} $k-1$}

                \If{$(t^n_i - t^n_{i+1}) \geq 2 \ \textbf{and} \ ((i = 1 \ \textbf{and} \ t^n_i \geq 1) \ \textbf{or} \ (t^n_{i} - t^n_{i-1}) \geq 2)$}
                \State $D_i^n \gets \left\{t^n_{i-1}+1, t^n_{i-1}+2,..., t^n_{i+1}-1\right\} \backslash \left\{ t^n_i \right\}$
                \For{$t$ \textbf{in} $D_i^n$}

                    \State $\Tilde{X}_t \gets \left(t-t^n_i,  \1_\mathrm{t > t^n_i}, \1_\mathrm{t > t^n_i} \times (t-t^n_i)  \right)$
                 \EndFor
                 
                \State $f \gets \text{LinearModel}.fit((\Tilde{X}_t)_{t \in D_i^n}, (Y_t)_{t \in D_i^n}, (K(t))_{t \in D_i^n})$
                \State Append $f.\text{coefficients}[2]$ to $S^{\text{CATE}}_{\text{test}}$ \Comment{Extract $\beta_2$ from the model $f$}            
            \EndIf
            
            \EndFor
        \EndIf
   
    \EndFor
    
    \State \textbf{Return} $S^{\text{CATE}}_{\text{test}}$
    \end{algorithmic}
    \label{alg:rdd_dataset}

\end{algorithm}

\subsection{TFT}
\label{sec:appendix_tft}
We developed a custom version of the Temporal Fusion Transformer \cite{TFT}. We focused on this architecture as it supports static features, past and futures features as input, and temporally causal processing of the input, where most architectures focus on predicting future time step from the previous ones \cite{itransformer},\cite{zhang2023crossformer},\cite{autoformer},\cite{PatchTST}.
We apply several changes to the original TFT:
\begin{itemize}
    \item We remove the encoder of the TFT and replace it by a series of encoder blocks. Figure \ref{fig:encoder_tft} presents the architecture of our encoder block. Each encoder block is made of self-attention and residual blocks. Each block divides the length of the input by 2. Hence, at the beginning of the encoder, the model processes high frequency features, and the deeper the encoder is the lower frequency the model is considering.
    \item We use conditional mechanisms that show interesting performances on Image Generation. Static features go through a dense model to create an embedding $s$ gathering all the static information. We then apply successive Cross Attention mechanisms between the token $s$ and the tokens of the temporal features \cite{rombach2022highresolutionimagesynthesislatent} to mix temporal and static features. Finally, we apply adaptive group normalization \cite{dhariwal2021diffusionmodelsbeatgans} where $s$ is the class embedding.
    \item We replace the masked self-attention decoder by a new decoder block. Figure \ref{fig:decoder_tft} shows the layers used in the decoder block. This block is almost the symmetric of the encoder one. We remove the cross-attention layer in the decoder as it increases the expressivity of the model. We also add skip connections between the encoder the decoder blocks as done in \cite{ronneberger2015unetconvolutionalnetworksbiomedical}. Thus, the information of high frequency features are directly transferred to the decoder without passing thought the entire encoder in order to help the expressivity of the model.
    \item We use causal convolution layers to prevent any possible leak of information from future features $X_{t+i}$ to impact $Y_t$ prediction. We also apply masking in our self-attention layers as done in \cite{li2020enhancinglocalitybreakingmemory}.
\end{itemize}

\begin{figure}[H]
\centering
\begin{subfigure}{.49\textwidth}
  \centering
  \includegraphics[scale=0.3]{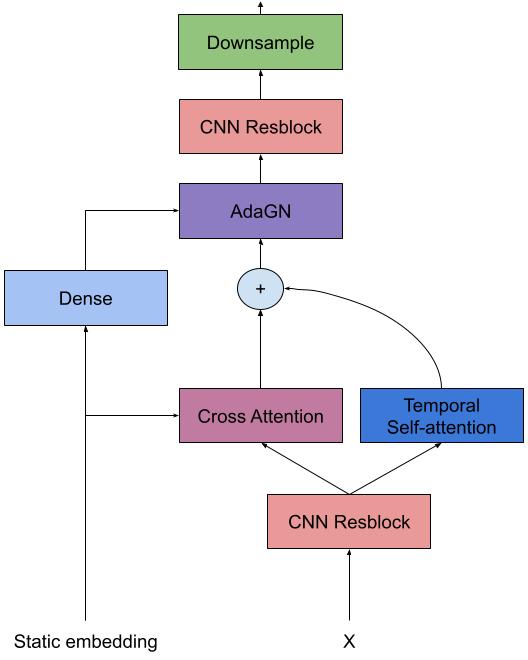}
  \caption{Encoder block}
  \label{fig:encoder_tft}
\end{subfigure}%
\hfill
\begin{subfigure}{.49\textwidth}
  \centering
  \includegraphics[scale=0.3]{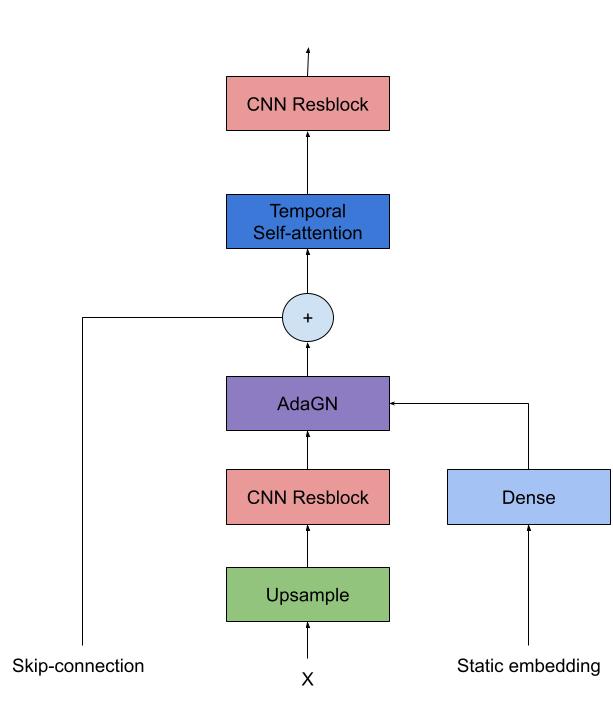}
  \caption{Decoder block}
  \label{fig:decoder_tft}
\end{subfigure}
\vspace{0.5em}
\caption{Our modified TFT architecture.}
\label{fig:Causal_TFT_arch}
\end{figure}

\section{Additional results}
\label{appendix:additional-results}

\subsection{Addition MIMIC-II results}
\label{appendix:mimic}
Table \ref{tab:MIMIC_table_MAE} shows MAE results for in-distribution forecasts. Table \ref{tab:MIMIC_CATE_MAE} shows MAE results for causal effects using our RDD dataset. Results show the same effects as those described in \S\ref{sec:eval-mimic} for the RMSE and RDD RMSE.

\begin{table}[H]
\begin{center}
\vspace{0.5em}
\resizebox{14cm}{!}{
\begin{tabular}{llllll}
\multicolumn{1}{c}{\bf Models}  &\multicolumn{1}{c}{ $t$ = $\tau$+1} &\multicolumn{1}{c}{ $t$ = $\tau$+2} &\multicolumn{1}{c}{ $t$ = $\tau$+3 } &\multicolumn{1}{c}{ $t$ = $\tau$+4 } &\multicolumn{1}{c}{ $t$ = $\tau$+5 }
\\ \hline \\
Causal Transformer         & \bf 6.028 $\pm$ 0.027& \bf 6.703 $\pm$ 0.027 & 7.048 $\pm$ 0.025 & 7.290 $\pm$ 0.032 & 7.485 $\pm$ 0.040\\
TFT Baseline         &6.162 $\pm$ 0.044  & 6.745 $\pm$ 0.024 & \bf 7.032 $\pm$ 0.018 & \bf 7.233 $\pm$ 0.022 & \bf 7.402 $\pm$ 0.024\\
Causal TFT One-hot        &6.170 $\pm$ 0.030  & 6.802 $\pm$ 0.033 & 7.115 $\pm$ 0.033 & 7.337 $\pm$ 0.036 & 7.515 $\pm$ 0.033\\
Causal TFT Cumulative        &6.147 $\pm$ 0.045  & 6.799 $\pm$ 0.054 & 7.125 $\pm$ 0.044 & 7.359 $\pm$ 0.05 & 7.544 $\pm$ 0.051\\
\end{tabular}
}
\caption{MAE with respect to the ground truth time series, best values are in \bf bold}
\label{tab:MIMIC_table_MAE}
\end{center}
\vspace{0.5em}
\end{table}

\begin{table}[H]
\begin{center}
\vspace{0.5em}
\resizebox{14cm}{!}{
\begin{tabular}{llllll}
\multicolumn{1}{c}{\bf Models}  &\multicolumn{1}{c}{ $t$ = $\tau$+1} &\multicolumn{1}{c}{ $t$ = $\tau$+2} &\multicolumn{1}{c}{ $t$ = $\tau$+3 } &\multicolumn{1}{c}{ $t$ = $\tau$+4 } &\multicolumn{1}{c}{ $t$ = $\tau$+5 }
\\ \hline \\
Causal Transformer         & 2.286 $\pm$  0.096  &  2.271 $\pm$ 0.097 & \bf 2.249 $\pm$ 0.096& \bf 2.241 $\pm$ 0.094 & \bf 2.242 $\pm$ 0.085\\
TFT Baseline         & 2.301 $\pm$ 0.065  & 2.320 $\pm$ 0.067 & 2.337 $\pm$ 0.071 & 2.368 $\pm$ 0.065 & 2.400 $\pm$ 0.058\\
Causal TFT One-hot        & \bf 2.265 $\pm$ 0.066  & \bf 2.265 $\pm$ 0.066 & 2.265 $\pm$ 0.066 & 2.265 $\pm$ 0.066 & 2.265 $\pm$ 0.066\\
Causal TFT Cumulative        &  2.292 $\pm$ 0.044  & 2.292 $\pm$ 0.044 & 2.292 $\pm$ 0.044 & 2.292 $\pm$ 0.044 &  2.292 $\pm$ 0.044\\
\end{tabular}
}
\caption{RDD MAE w.r.t the CATE, best values are in \bf bold}
\label{tab:MIMIC_CATE_MAE}
\end{center}
\vspace{0.5em}
\end{table}

\section{Extended Related Work}
\label{sec:relwork}
Our two background sections \S\ref{sec:causal-background} and \S\ref{sec:rdd-background} already discuss the closest related work. In this section, we expand on this discussion and discuss other related papers.

{\bf Time-series forecasting} has seen a flurry of recent activity. The main task studied in the literature is that of multivariate time series forecasting, with the goal of forecasting $(Y^n_i)_{i \in [t:t+h]}$ from $(Y^n_i)_{i \in [0:t]}$. Another common task is to forecast $(Y^n_i)_{i \in [1:t]}$ from $(X^n_i)_{i \in [1:t]}$. Various approaches have been proposed, though not clear winner has emerged yet. TSMixer \cite{TSMixer} relies only on fully connected linear layers to process entire time-series feature-wise, before mixing features at each time-step. The iTransformer \cite{itransformer} architecture similarly leverages fully connected layers, but processes intermediate representations with transformer layers channel wise. PatchTST \cite{PatchTST} gathers patches time-steps and features that it treats as tokens in a self-attention mechanism, before reconstructing a time-series shaped output. In all these architectures and others \cite{SAMFormer,autoformer}, there is no insurance against future information leaking into previous time-steps. While this can be fine in some cases, this is not compatible with our specific task.
Indeed, our causal models predict $Y^n_{t+l}$ based on past features $(Y^n_i)_{i \in [0:t]}$, static features $S^n$ and temporal features $(X^n_i)_{i \in [0:t+l]}$. A key point in our task is to ensure causality in the forecasting, preventing any leak of information from $(X^n_i)_{i \in [t+1:t+h]}$. Supporting this requirement is not common in recent time series papers, and the reason why we build on the TFT \cite{TFT}, which does provide an adequate structure. 

{\bf Causal time-series models} have seen less progress, though some time-series models aim to estimate the CATE.
\citet{timeseriesdeconfounderestimating,causal_forecasting} rely on the same meta-architecture, including a network computing the treatment propensity score and a network computing the outcome (though \cite{timeseriesdeconfounderestimating} tries to learn unobserved confounders, a challenging task). Both models use LSTM or RNN architectures merged with new mechanisms: a propensity network learns to predict propensity scores, which are then used to train the outcome network with the inverse probability of treatment weighting. This loss, mainly used for policy learning \cite{foster2023orthogonal}, can be high variance and does not full corrects for regularization bias. This is why we decided to train our model with the Residualized Loss (R-Loss) \cite{foster2023orthogonal,Nie_Wager}, which also gives us a more flexible approach for treatment encoding.

\citet{Brodersen_2015} take a different approach, and use structured time-series models to explicitly model counterfactuals under no-treatment as a synthetic control. This approach is well suited to binary treatments with several observations, to be able to model outcomes without treatment. It is less applicable to our case with high dimensional treatments and complex conditioning variables. However, it could interesting to explore this approach as a multi-time-steps alternative to our RDD causal test set.

With the recent rise of the transformer architecture for sequential models, new techniques have been developed to merge transformers with causal inference \cite{causal_transformers}. The authors leverage an adversarial loss \cite{adversarial_loss} to encourage the model to learn balanced features for the treated and untreated populations and learn the causal link between the treatment and the outcome.
\citet{causal_transformers} use a transformer backbone with two small dense layers on top trained with this adversarial loss.
This identification approach differs from classical causal inference frameworks, and it is unclear that it converges to a well defined conditional causal effect.

{\bf Regression Discontinuity Designs (RDDs)} are studied theoretically, and frequently used empirically, in economics \cite{hahn2001identification,imbens2008regression,lee2010regression}.
Researchers typically use RDDs to estimate one specific ATE. In contrast, we estimate many ATEs (one for each time series), and combine them with time-series features $W_t$ to create a CATE dataset. To the best of our knowledge, this is a novel appraoch to evaluate causal models.
Researchers have used RDDs to estimate ATEs in time-series before though. \citet{hausman2018regression} surveys typical approaches and pitfalls, and proposes empirical checks to perform on the data to verify if RDDs are applicable. In our setting, we typically do not have enough data around switching times to perform such checks, though it would be interesting to take inspiration from common RDD checks to improve our causal test sets. We leave this direction for future work.

\end{document}